\documentclass[letterpaper, 10 pt, journal, twoside]{IEEEtran}

\let\proof\relax
\let\endproof\relax
\usepackage[normalem]{ulem}
\usepackage{eso-pic}
\usepackage{amsthm}
\usepackage[utf8]{inputenc}
\usepackage{amsmath,amssymb}
\usepackage{booktabs}
\usepackage{graphicx}
\usepackage[font=small]{subcaption}
\usepackage{numprint}
\usepackage{comment}
\usepackage[linesnumbered,ruled,noend]{algorithm2e}
\usepackage{algorithmic}
\usepackage{eso-pic}
\usepackage[american]{babel}
\usepackage{svg}
\usepackage{xspace}
\usepackage[subtle, tracking=normal, wordspacing=normal]{savetrees}

\usepackage{xcolor}

\SetCommentSty{mycommfont}

\usepackage{adjustbox}
\usepackage{tikz}
\usetikzlibrary{automata,positioning,shapes,arrows}
\usepackage[normalem]{ulem} % only for \zsedit below
\usepackage{tabularx,siunitx}
\usepackage{cite}
\usepackage{hyperref}
\usepackage{multirow}

\newtheorem{theorem}{Theorem}
\newtheorem{problem}{Problem}

\newtheorem{lemma}{Lemma}
\newtheorem{definition}{Definition}
\newtheorem{remark}{Remark}

\newcommand{\rv}[1]{{#1}}

\newcommand{\alg}{\textit{Kino-PAX}\xspace}
\newcommand{\itr}{50\xspace}
\newcommand{\maxT}{60\xspace}
\newcommand{\free}{\text{valid}}
\newcommand{\goal}{\text{goal}}
\newcommand{\init}{\text{init}}

\begin{document}

\title{Kino-PAX: Highly Parallel Kinodynamic Sampling-based Planner}

\author{Nicolas Perrault, Qi Heng Ho, and Morteza Lahijanian % <-this % stops a space
\thanks{Manuscript received: September 8, 2024; Revised: December 4, 2024; Accepted: January 2, 2025.}% <-this % stops a space
\thanks{This paper was recommended for publication by Editor Lucia Pallottino upon evaluation of the Associate Editor and Reviewers’ comments.}% <-this % stops a space
% \thanks{This work was supported by NASA STTR award 80NSSC20C0314.}
% \thanks{$^*$A. Theurkauf and J. Kottinger have equal contributions in this work.}
\thanks{Authors are with the department of Aerospace Engineering Sciences at the University of Colorado Boulder, CO, USA
        {\tt\small \{\textit{firstname}.\textit{lastname}\}@colorado.edu}}
\thanks{Digital Object Identifier (DOI): see top of this page.}
}

% Paper headers
% \markboth{IEEE Robotics and Automation Letters. Preprint Version. Accepted January, 2025}
% {Perrault \MakeLowercase{\textit{et al.}}: Kino-PAX: Highly Parallel Kinodynamic Sampling-based Planner} 
\AddToShipoutPictureBG*{%
  \AtPageUpperLeft{%
    \hspace{16.5cm}%
    \raisebox{-1.1cm}{%
      \makebox[0pt][r]{To appear in the Robotics and Automation Letters (RAL), March 2025.}}}}

% Use only for final RAL version

%\IEEEpubid{0000--0000/00\$00.00~\copyright~2021 IEEE}
% Remember, if you use this you must call \IEEEpubidadjcol in the second
% column for its text to clear the IEEEpubid mark.

\maketitle

% Abstract
\begin{abstract}
    Sampling-based motion planners (SBMPs) are effective for planning with complex kinodynamic constraints in high-dimensional spaces, but they still struggle to achieve \emph{real-time} performance, which is mainly due to their serial computation design.
    We present \textit{Kinodynamic Parallel Accelerated eXpansion} (\alg), a novel highly parallel kinodynamic SBMP designed for parallel devices such as GPUs. \alg grows a tree of trajectory segments directly in parallel. Our key insight is how to decompose the iterative tree growth process into three massively parallel subroutines. \alg is designed to align with the parallel device execution hierarchies, through ensuring that threads are largely independent, share equal workloads, and take advantage of low-latency resources while minimizing high-latency data transfers and process synchronization. This design results in a very efficient GPU implementation. We prove that \alg is probabilistically complete and analyze its scalability with compute hardware improvements. Empirical evaluations demonstrate solutions in the order of $10$ ms on a desktop GPU and in the order of $100$ ms on an embedded GPU, representing  up to $1000\times$ improvement compared to coarse-grained CPU parallelization of state-of-the-art sequential algorithms over a range of complex environments and systems.
\end{abstract}

\begin{IEEEkeywords}
Constrained Motion Planning, Motion and Path Planning, Computer Architecture for Robotic and Automation.
\end{IEEEkeywords}

%% SECTIONS:
 \section{Introduction}
    \label{sec:intro}

\IEEEPARstart{A}{utonomous} robotic systems are increasingly deployed in dynamic environments, requiring fast, reactive motion planning that accounts for the robot's complex kinematics and dynamics. Solving the kinodynamic motion planning problem \emph{quickly} is critical for ensuring both functionality and safety. \textit{Sampling-based motion planners} (SBMPs) have proven effective for various difficult problems, such as complex dynamics \cite{lavalle2001randomized, phillips2004guided, plaku2010motion, ladd2005fast,  sucan2011sampling}, complex tasks \cite{bhatia2010sampling, Maly:HSCC:2013, Lahijanian:TRO:2016}, and stochastic dynamics \cite{luders2010chance, 
% huynh2012incremental, 
% Luna:AAAI:2014,Luna:WAFR:2015,
ho:2022:ICRA,Ho:ICRA:2023}.
Nevertheless, they are typically designed for serial computation, limiting their speed to CPU clock rate.
% \qh{I'm not sure if the argument is complete. Just because they are designed for serial computation doesn't immediately make them limited}. 
While recent methods can find solutions within seconds for simple systems and tens of seconds for complex ones~\cite{lavalle2001randomized, phillips2004guided, plaku2010motion}, this is insufficient for \emph{real-time} reactivity. Given the plateau in improvements to serial computation and CPU clock speeds, parallel devices like GPUs offer promising speedups. However, current SBMP algorithms are inherently sequential and inefficient when parallelized. 
In this work, we aim to enable real-time motion planning for complex and high-dimensional kinodynamical systems by exploiting the parallel architecture of GPU-like devices.

In this paper, we introduce \alg, a highly parallel kinodynamic SBMP, designed to efficiently leverage parallel devices. \alg grows a tree of trajectory segments directly in parallel. Our key insight is that the iterative tree growth process can be decomposed into three massively parallel subroutines. 
We design \alg to align with the parallel execution hierarchy of these devices, ensuring that threads are largely independent, share equal workloads, and take advantage of low-latency resources while minimizing high-latency data transfers and process synchronizations. We provide an analysis of \alg, showing that it is probabilistically complete.  We also demonstrate, through several benchmarks, that \alg is robust to changes in hyperparameters and scalable to large dimensional systems.

In summary, our contributions are four-fold: (i) \alg, a highly parallel kinodynamic SBMP designed to leverage the parallel architecture of GPU-like devices, (ii) \rv{a discussion of \alg's hyperparameters and its efficient application to highly parallel devices}, (iii) a thorough analysis and proof of probabilistic completeness, and (iv) benchmarks showing the efficiency and efficacy of \alg for complex and high-dimensional dynamical systems. Our results show that \alg achieves up to three-orders-of-magnitude improvement in computation time compared to our baselines, which use CPU parallelization. 
% over $3$ dynamical systems. 
In all evaluated problems, \alg finds solutions in the order of $10$ milliseconds, representing significant progress in enabling real-time kinodynamic motion planning.

\subsection{Related Work}
    \label{sec:related}
    \paragraph*{Geometric Motion Planning}
% 
% \ml{add and discuss planning in microseconds paper}\np{\cite{thomason2024motions} is the microsecond paper. I could elaborate further on it?}
% \ml{yes!}
SBMPs have a long-standing history in addressing the geometric motion planning problem, as established by foundational works such as \textit{Probabilistic RoadMaps} (PRM)~\cite{kavraki1996probabilistic}, \textit{Expansive-Space Tree} (EST)~\cite{hsu1997path}, \textit{Rapidly-exploring Random Tree} (RRT) \cite{lavalle2001rapidly}, etc. In general terms, SBMP techniques involve finding a path from a starting configuration to a goal region by constructing a graph or tree, where nodes represent geometric configurations and straight line edges represent transitions in the configuration space. Traditionally, these algorithms operate serially on CPU devices. However, the increasing demand for rapid replanning for complex systems in unknown and dynamic environments has driven the development of parallelization methods for geometric SBMPs. These parallelized approaches have been applied to both CPU-based planners \rv{\cite{otte2013c, plaku2005sampling, sun2015high, ichnowski2012parallel,ichnowski2020concurrent, thomason2024motions}} and GPU-based implementations \cite{ichter2017group}. 

\rv{The CPU-based methods in \cite{otte2013c, plaku2005sampling, sun2015high} use coarse-grained parallelization, where threads independently construct trees and exchange search information. These methods offer straightforward implementation by leveraging existing serial SBMPs but they achieve limited improvements in planning rates, and are not applicable to many-core devices. More similar to our work, the approaches in \cite{ichnowski2012parallel, ichnowski2020concurrent, thomason2024motions} focus on fine-grained parallelization to accelerate the construction of a single tree. Works \cite{ichnowski2012parallel, ichnowski2020concurrent} introduce parallel RRT and RRT* methods that leverage thread-safe atomic operations and a novel concurrent data structure for efficient nearest neighbor search. In contrast, \cite{thomason2024motions} decomposes core operations of geometric SBMPs into unconventional data layouts that enable parallelism without specialized hardware. While these works present promising fine-grained parallelization techniques, they are not easily adaptable to kinodynamic planning and are unsuitable for many-core technology due to high inter-thread communication costs.}

% Work \cite{thomason2024motions} decomposes fundamental operations of geometric SBMPs, such as collision checking, into unconventional data layouts, enabling parallelism without requiring specialized hardware. Works \cite{otte2013c, plaku2005sampling, sun2015high} propose growing multiple trees in parallel. The algorithm in \cite{otte2013c} shares tree information between threads, improving the visibility of the search space. The method in \cite{plaku2005sampling} generates trees from different regions of the planning space, attempting to connect their branches. \cite{sun2015high} runs multiple RRT instances in parallel, selecting the most optimal solution. The works \cite{ichnowski2012parallel, ichnowski2020concurrent} present a fine-grained parallelization method for RRT and RRT*, utilizing thread-safe operations and a concurrent data structure for nearest neighbor search and node insertion. 

\rv{Most similar to our work, \cite{ichter2017group} leverages GPUs to solve the geometric path planning problem by adapting \textit{FMT*} \cite{janson2015fast} for parallelized graph search. Their implementation achieves orders-of-magnitude performance improvements over its serial counterpart; however, it is strictly limited to geometric problems.}
% The algorithm in \cite{ichter2017real} takes GPU-based geometric motion planning a step further by constructing a Pareto set of motion plans, which considers both collision probability and dynamic feasibility when calculating a path.
% \rv{Lastly, the methods presented in \cite{DBLP:journals/corr/abs-1710-03937, DBLP:journals/corr/abs-1902-09458} offer long-range motion planning through a highly parallelizable PRM construction phase and a reinforcement learning-based steering function.}
% While these algorithms have demonstrated orders of magnitude improvement over serial implementations, they are strictly limited to geometric problems.

% in \cite{thomason2024motions,otte2013c}, and \cite{ichter2017group,ichter2017real} require solving two-point boundary value problems, making them inapplicable to general dynamical systems.

\paragraph*{Kinodynamic Motion Planning}
To provide dynamically feasible and collision-free trajectories for systems with complex dynamics, a kinodynamic motion planning algorithm is used, as seen in traditional serial solutions such as \cite{lavalle2001randomized, phillips2004guided,  plaku2010motion, ladd2005fast, sucan2011sampling}. The details of the kinodynamic motion planning problem are formally discussed in Section \ref{sec:problem}; however, in general, these algorithms are tree-based and solve the problem by sequentially randomly extending trajectories until a path from a start state to a goal region satisfying all state constraints can be followed by a sequence of trajectory segments.

To achieve fast planning times, works \cite{plaku2010motion,sucan2011sampling} employ space discretization. Specifically, \cite{plaku2010motion} constructs a graph from discrete regions and uses it as a high-level planner to guide the motion tree.  However,
this approach can face the \emph{state-explosion} problem as the dimensionality of state space increases.
In contrast, \cite{sucan2011sampling} avoids this issue by using the discrete regions to track spatial information about the sparsity of the motion tree without constructing a graph.  In the design of \alg, we take inspirations from those planners, using discrete regions to guide the search. Similar to \cite{sucan2011sampling}, we use these regions for spatial information to ensure scalability. However, unlike previous work, \alg performs these operations in parallel subroutines. 

In contrast to the parallelized geometric planning solutions discussed above, parallelization for planning under the constraints of general dynamical systems is relatively understudied in the field. An approach to parallelization for the kinodynamic problem is a coarse-grained method, where multiple trees of classical SBMPs are generated in parallel, and the first solution found is returned \cite{746692}. This technique improves average-case performance \cite{wedge2008heavy}, and is employed in Section \ref{sec:experiments} to perform CPU-based parallelization as a baseline. However, the inherent sequential nature of these motion planners make them inefficient for massive parallelization. In this work, we propose a novel algorithm that enables efficient application to highly parallel devices.

% \ml{In design of our planner, we take inspirations from geometric planner in \cite{ichter2017group} and kinodynamic planners in \cite{plaku2010motion,sucan2011sampling}.  Talk about similartities blah blah... we avoid the problem of state-space explosion since we do not construct a discrete graph.}

% comparison against our novel GPU parallelization method

\section{Problem Formulation}
    \label{sec:problem}
    % For this work, we consider the problem of motion planning under dynamical constraints, where the evolution of the agent's state is described by a differential equation of the form:

% \begin{equation*}
%     \dot x = f(x, u)
% \end{equation*}

% where:
% \begin{itemize}
%     \item $x \in X$ is the state, and $X$ is the state space, $X \subset \mathbb{R}^n$.
%     \item $u \in U$ is the input control, and $U$ is the control space, $U \subset \mathbb{R}^l$.
%     \item $f(\cdot,\cdot)$ is a continuous integrable function.
% \end{itemize}

% Additionally, a trajectory $\gamma$ generated by applying $n$ consecutive control inputs, each for a duration $t_i$, where the total time is given by $T = \sum_{i=1}^{n} t_i$, is valid and meets the goal criteria if and only if:

% \begin{itemize}
%     \item $\forall t \in [0,T] : \gamma(t) \in X_{safe}$, where $X_{safe} \subseteq X$ is the set of states satisfying geometric and dynamical constraints.
%     \item $\gamma(T) \in X_{goal}$, where $X_{goal} \subseteq X$ is the set of states satisfying the goal region criteria.
% \end{itemize}

% In this work, we address the kinodynamic motion planning problem, which involves finding a geometrically and dynamically feasible trajectory from an initial state to a goal region. 

Consider a robotic system operating within a bounded workspace $W \subset \mathbb{R}^d$, where $d \in \{2, 3\}$. This workspace contains a finite set of obstacles $\mathcal{O}$, where each obstacle $o \in \mathcal{O}$ is a closed subset of $W$, i.e., \rv{$o \subset W$}.
The dynamics of the robot's motion is given by
% the following differential equation:
\begin{equation}
\label{eq:diffEq}
\dot{x}(t) = f(x(t), u(t)), 
% \quad x(t) \in X \subseteq \mathbb{R}^n , \quad u(t) \in U \subseteq \mathbb{R}^N,
\end{equation}
where 
$x(t) \in X \subset \mathbb{R}^n$ and $u(t) \in U \subset \mathbb{R}^N$ are the robot's state and control at time $t$, 
% $X$ and $U$ are the robot's state and input spaces, 
respectively, and $f:X \times U \to \mathbb{R}^n$ is the vector field.  We assume that $f$ is a Lipschitz continuous function with respect to both arguments, i.e, there exist constants \(K_x, K_u > 0\) such that for all \(x, x' \in X\) and \(u, u' \in U\),
$$\|f(x, u) - f(x', u')\| \leq   K_x \|x - x'\| + K_u \|u - u'\|.$$

% \ml{is it $x \in X$ or $x \in X$?  Also, what are $X,X$?}
% where $x \in X$ represents the state, $u \in U$ denotes the control input, and $f(\cdot, \cdot)$ is a continuous, integrable function that expresses the system's dynamics. Let $X \subseteq \mathbb{R}^d$ 
% \ml{oh! you define the sate space here. It should rather be defined first time it's use above. You could say ``$x \in X \subseteq \mathrm{R}^n$ is the state.'' Then it is understood that $X$ is the state space and it's $n$-dimensional.}
% denote the state space that is a smooth $d-dimensional$ manifold, and let $U \subseteq \mathbb{R}^D$ 
% \ml{same comment for $U$}
% denote the space of control vectors.

% Additionally, we define the set of states that satisfy the geometric and dynamic constraints is represented by $X_{safe} \subseteq X$
% \ml{how could a state satisfy the dynamic constraints?  System dynamics (differential equations) pose constraints on the time evolution of state.  If by dynamic constraints you mean state constraints, e.g., bound on velocity, then I think you should use a better terminology than ``dynamic'' since it could be confusing.
% See the problem formulation of this paper as and example how to define kinodynamic motion planning problem: \url{https://arxiv.org/pdf/2207.00576}
% }
% . The agent's initial condition is given by $x_{init} \in X_{safe}$, and the goal region is denoted by $X_{goal} \subseteq X_{safe}$.

In addition to motion constraints defined by the dynamics in~\eqref{eq:diffEq} and obstacles in $\mathcal{O}$, we consider state constraints, e.g., bound on the velocity.
To this end, we define the set of valid states, i.e., states at which the robot does not violate its state constraints and does not collide with an obstacle, as the \rv{valid} set and denote it by $\rv{X_{\free}} \subseteq X$.  
Then, given initial state $x_{\init} \in \rv{X_{\free}}$, time duration \(t_{f} \geq 0\), and control trajectory $\mathbf{u}: [0,t_{f}] \to U$, 
a \textit{state trajectory} $\mathbf{x}: [0,t_{f}] \to X$ is induced, where
\begin{align}
    % \mathbf{x}(0) &= x_{\init},\\
    \mathbf{x}(t) = x_{\init} + \int_{0}^{t} f(\mathbf{x}(\tau),\mathbf{u}(\tau)) d \tau \qquad \forall t \in [0,t_f].
    \label{eq:integral}
\end{align}
Trajectory $\mathbf{x}$ is called \textit{valid} if, $ \forall t \in [0,t_f]$, $x(t) \in \rv{X_{\free}}$.
% \begin{align*}
%     f(0) &= x_{\init}, \\
%     f(t') &= \int_{0}^t f(x(t),u(t)).
% \end{align*}

% . 

In motion planning, the interest is to find a valid trajectory $\mathbf{x}$ that visits a given goal set $X_\goal \subseteq \rv{X_{\free}}$.
% , i.e., $\mathbf{x}$ is valid and $\mathbf{x}(t) \in X_{\goal}$ for some $t \in [0,t_f]$.  
Therefore, by following this trajectory, the robot is able to respect all of its motion (kinodynamic) constraints, avoid collisions with obstacles, and reach its goal. 
% While there exist many algorithms that can solve this problem, they are often slow and cannot be used for online planning to provide, e.g., needed reactivity to deal with changing environments.  
In this work, we focus on kinodynamic motion planning with an emphasis on computational efficiency through parallelism.

\begin{problem}[Kinodynamic Motion Planning]
    \label{problem}
    Consider a robot with dynamics in \eqref{eq:diffEq} in workspace $W$ consisting of obstacle set $\mathcal{O}$.  Given an initial state $x_{\init} \in \rv{X_{\free}} \subseteq X$ and goal region $X_\goal \subseteq \rv{X_{\free}}$, \emph{efficiently} find a control trajectory $\mathbf{u}:[0,t_f] \to U$ such that its induced trajectory $\mathbf{x}$ through \eqref{eq:integral} is valid and reaches goal, i.e., $\mathbf{x}(0) = x_{\init}$ and
    \begin{align*}
        % \mathbf{x}(0) &= x_{\init}, && \\
        \mathbf{x}(t) &\in \rv{X_{\free}} && \forall t \in [0,t_f], \\
        \mathbf{x}(t) &\in X_\goal && \exists t \in [0,t_f].
    \end{align*}
\end{problem}

Note that this is a challenging problem.
The simpler problem of geometric motion planning (by ignoring dynamics) is already PSPACE-complete \cite{reif1979complexity}, and the addition of kinodynamic constraints makes finding a solution considerably more difficult due to the increase in search space dimension and dynamic complexities \cite{donald1993kinodynamic,laumond1993controllability}. Existing 
% serial implementations 
algorithms
find solutions in the order of seconds for simple (e.g., linear) systems and tens of seconds for more complex non-linear systems \cite{plaku2010motion, phillips2004guided, lavalle2001randomized} on standard benchmark problems. When combined with the need for fast replanning in, e.g., unknown and changing environments, finding solutions in real-time (milliseconds) becomes crucial for ensuring the functionality and safety of autonomous systems.

With the availability of onboard GPUs, parallel computation provides a promising approach for finding solutions quickly. 
% However, classical kinodynamic motion planners cannot take full advantage of GPUs, due to their inherently sequential subroutines that have strong interdependence. 
% We aim to achieve efficiency via a highly parallel algorithm that can take advantage processes with highly parallel architecture, e.g., GPUs.
% We aim to achieve efficiency through a highly parallelizable algorithm capable of leveraging processors with advanced parallel architectures, such as GPUs.
Hence, in our approach, we focus on achieving efficiency through a highly parallelizable algorithm.
% capable of leveraging processors with advanced parallel architectures, such as GPUs.

% \ml{another good discussion but I'd frame it differently.  I'd say CUDA-enabled GPUs are widely available these days, enabling parallel computation.  However, classical motion planners due to their design cannot take advantage of the computational power that GPUs offer. Our goal in this paper is to design an algorithm that fully exploits the parallelism and many-core GPU architecture...  Then, state the problem, which should include ``parallel computing'' or ``GPU'' in it.}
% Our approach to efficiently solving the kinodynamic motion planning problem leverages CUDA-enabled GPUs, which require careful algorithm design to fully exploit the parallelism and many-core GPU architecture.

\section{Kino-PAX}
    \label{sec:alg}
    
% The efficiency of the search in KGMT depends on the tree's ability to rapidly expand in all space dimensions. To achieve this, KGMT employs a high-level discrete representation of the tree’s exploration progress in each region, effectively guiding its expansion in all directions. When the low-level motion planner has accurate data on search progress, the many-core architecture of the GPU can be better leveraged to outwardly expand the tree as a massively parallel process. 

% \begin{figure}[t]
%     \centering
%     \includegraphics[width=0.98\columnwidth]{figures/KGMT_flowChart.drawio (4).pdf}
%     \caption{Overview of \alg.}
%     \label{fig:G-EST}
% \end{figure}

\begin{figure*}[ht!]
\label{test}
    \centering
    \begin{subfigure}[t]{0.07\textwidth}
        \centering
        \vspace{-36mm}
        \includegraphics[width=\textwidth]{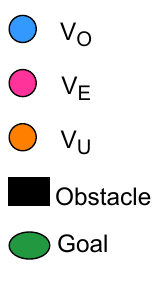}
    \end{subfigure}
    % ~
    \begin{subfigure}[b]{0.17\textwidth}        
    \centering
        \includegraphics[width=\textwidth]{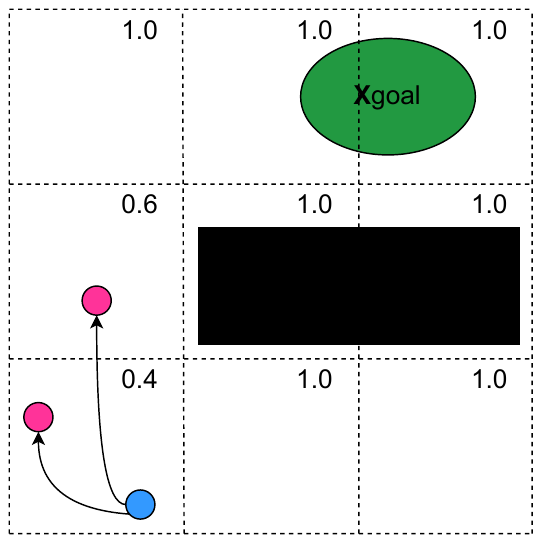}
        % \caption{$V_E$ for extension}
         \caption{}
        \label{fig:iteration1}
    \end{subfigure}
    ~
    \begin{subfigure}[b]{0.17\textwidth}
        \centering
        \includegraphics[width=\textwidth]{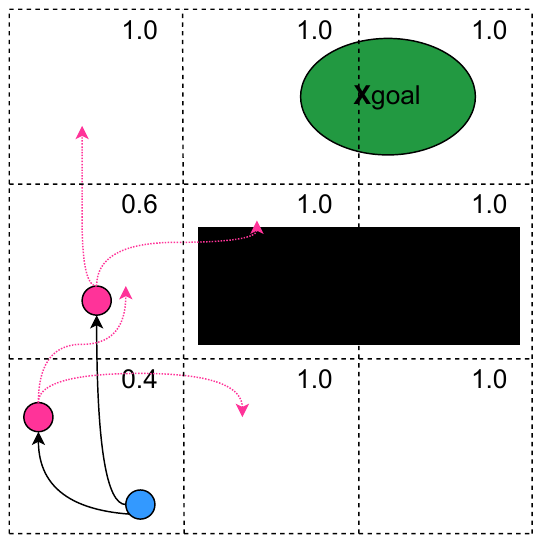}
        % \caption{Extend $V_E$ nodes}
        \caption{}
        \label{fig:iteration2}
    \end{subfigure}
    ~
    \begin{subfigure}[b]{0.17\textwidth}
        \centering
        \includegraphics[width=\textwidth]{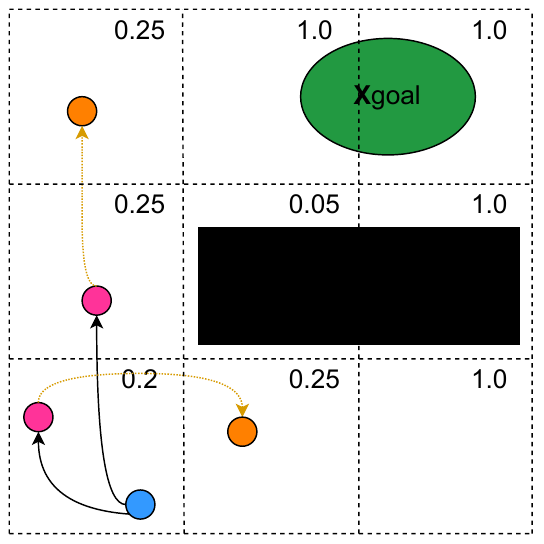}
        % \caption{Update $V_U$}
        \caption{}
        \label{fig:iteration3}
    \end{subfigure}
    ~
    \begin{subfigure}[b]{0.17\textwidth}
        \centering
        \includegraphics[width=\textwidth]{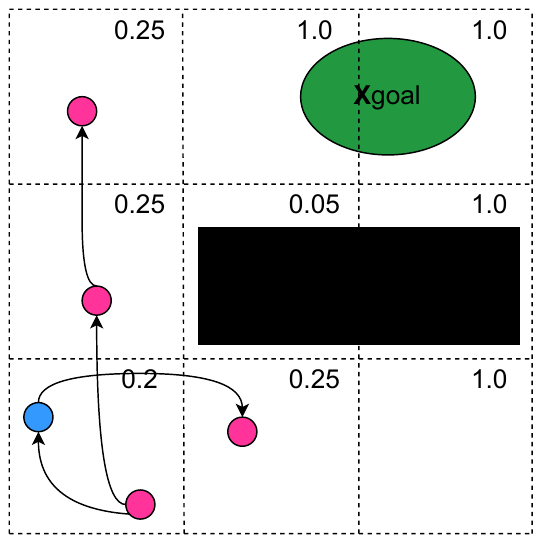}
        % \caption{Update $V_E$ and $V_O$}
        \caption{}
        \label{fig:iteration4}
    \end{subfigure}
    ~
    \begin{subfigure}[b]{0.17\textwidth}
        \centering
        \includegraphics[width=\textwidth]{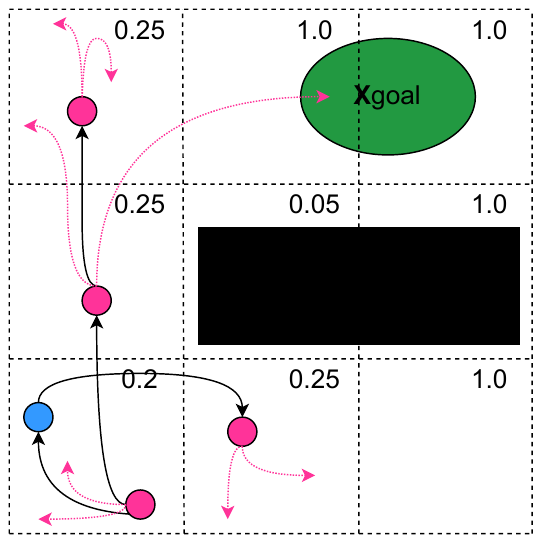}
        % \caption{Update $V_E$ and $V_O$}
        \caption{}
        \label{fig:iteration5}
    \end{subfigure}
    % \caption{
    % Illustration of a single iteration in the \alg expansion process. (a) The current sets $V_E$ and $V_{O}$, where the numbers in each grid cell represent the value of $P_{accept}(\mathcal{R}_i)$. (b) Expansion of each node in $V_E$ with a branching factor of $\lambda = 2$. (c) Acceptance of promising samples and the update of $P_{accept}(\mathcal{R}_i)$. (d) The updated set $V_E$, prepared for expansion in the next iteration.}
    % \label{fig:iteration_grid}
    \caption{ \rv{ Illustration of the \alg expansion process:  
(a) The current sets $V_E$ and $V_{O}$, where the numbers in each grid cell represent the value of $P_{accept}(\mathcal{R}_i)$.  
(b) Expansion of \( V_E \) with branching factor \( \lambda = 2 \).  
(c) Acceptance of promising samples and update of \( P_{\text{accept}}(\mathcal{R}_i) \).  
(d) Updated \( V_E \), ready for the next iteration. 
(e) Expansion of \( V_E \), producing a valid trajectory to \( X_{\text{goal}} \).
}}
    \label{fig:overview}
\end{figure*}

Our approach to Problem~\ref{problem} is a highly parallel algorithm that is able to exploit the many-core architecture of GPU-like processors. To achieve efficient performance on these high-throughput devices, it is crucial that our algorithm complements the execution hierarchy of such processors to optimize resource utilization. For the development of this algorithm, we follow the guidance of \cite{kirk2016programming,guide2020cuda} and base our development on three key principles: 
(i) \emph{thread independence}, the ability for each thread in a program to execute without being dependent on the state or result of other threads; 
(ii) \emph{even workloads across threads}, each thread is assigned an equal or nearly equal number of operations throughout its execution; 
(iii) \emph{utilization of low-latency memory}, groups of threads utilize low-latency memory to share information and reduce the number of higher-latency global memory accesses.

% \begin{itemize} 
%     \item \emph{Thread independence} — The ability for each thread in a program to execute without being dependent on the state or result of other threads.
%     \item \emph{Even workloads across threads} — Each thread is assigned an equal or nearly equal number of operations throughout its execution.
%     \item \emph{Utilization of low-latency memory} —  Groups of threads utilize low-latency memory to share information and reduce the number of higher-latency global memory accesses.
% \end{itemize}

With these principles in mind, we introduce \emph{Kinodynamic Parallel Accelerated eXpansion} (\alg), a highly parallel kinodynamic SBMP. 
% An overview of \alg is shown in Fig.~\ref{fig:G-EST}.
\alg grows a tree of trajectory segments in parallel. This is achieved by decomposing the iterative tree growth process, i.e., selection of nodes, extension, validity checking, and adding new nodes to the tree, into three massively parallel subroutines. Each subroutine follows the key principles of \emph{thread independence}, \emph{balanced workloads}, and \emph{low-latency memory utilization}. Additionally, to ensure fast and efficient planning iterations, 
% we minimize CPU-GPU communication during synchronization steps between subroutines.
we minimize the communication needed in the synchronization steps between subroutines, e.g., CPU-GPU communication.

At each iteration of \alg, a set of nodes in the tree is expanded in parallel. Each sample is extended multiple times through random sampling of controls, also in parallel. We dynamically adjust the number of extensions in each iteration to maintain an effective tree growth rate, ensuring efficient usage of the device's throughput. After extension, a new set of nodes are selected independently to be propagated in the next iteration.

% To guide the search process, \alg employs a high-level space decomposition approach that is well suited for parallel computation. The decomposition estimates the exploration progress made in each region, helping to identify promising tree nodes for expansion and acceptance of new samples in less visited space. This decomposition allows threads in \alg to act independently when adding new nodes to the tree and when assessing whether an existing node is promising for extension. As more trajectory segment data becomes available, the estimate of promising space regions is improved, allowing \alg to focus on propagating a large number of favorable samples into less explored areas of the space.

To guide the search process, \alg, similar to \cite{plaku2010motion,sucan2011sampling}, employs a high-level space decomposition approach. We designed this method to be well-suited for parallel computation. This decomposition estimates exploration progress in each region, allowing threads to act independently when adding new nodes to the tree and identifying promising nodes for extension. As more trajectory segment data becomes available, the estimate of promising space regions is improved, allowing \alg to focus on propagating a large number of favorable nodes into less explored areas of the space.

% By directing which tree samples to propagate from in future parallel expansions, the algorithm effectively manages the growth of the tree.

% In early iterations, when the tree is relatively small, each sample undergoes numerous extensions to rapidly expand the search space and jumpstart the tree growth process. As the size of the tree increases, \alg reduces the number of extensions per sample to prevent the serialization of operations due to device throughput limitations. 

% To guide this process, a high-level space decomposition is used to estimate the exploration progress within each discrete region of the space. This decomposition identifies promising regions for new sample acceptance and directing future node expansions. As the tree iteratively grows in a massively parallel process, new information on valid and invalid trajectories guides \alg in determining where to focus the next parallel expansion.

% The design of \alg considers the key points discussed in Section \ref{sec:CUDA_computing} to develop a novel kinodynamic motion planning algorithm that efficiently utilizes GPU resources. 

% \alg operates by generating trajectory segments in parallel, using a set of samples known as the frontier. Each sample is extended a number of  times to better utilize the device's throughput. \alg operates by expanding multiple nodes in parallel, simultaneously branching the tree in all relevant directions within the search space.

% The resulting structure resembles a breadth-first search, with branches flooding outward in all directions of the state space simultaneously.

\subsection{Core Algorithm}

Here, we present a detailed description of \alg. Pseudocode of \alg is presented in Alg.~\ref{algo: alg}, with subroutines Algs.~\ref{algo: Propagate}, \ref{algo:UpdateEstimates}, and \ref{algo:UpdateNodeSets}, 
% A visual overview of \alg is in Fig.~\ref{fig:G-EST}, 
and a full planning iteration is illustrated in Fig.~\ref{fig:overview}. \alg organizes samples into three distinct sets: $V_U, V_O, V_E$. The set $V_{U}$ consists of newly generated promising samples that have not yet been added to the tree. The set $V_{O}$ includes tree nodes that are not currently considered for expansion; intuitively, these nodes are located in densely populated or frequently invalid regions of the search space. Finally, $V_E$ comprises the set of nodes that are flagged for parallel expansion. Further, \alg maintains the spatial search progress information using a partition of the state space denoted by $\mathcal{R}$.
% \footnote{The space decomposition can also be made in the workspace, but it often leads to inefficiency as examined in Sec.~\ref{sec:experiments}. \np{need to add comparison metric}}
% \ml{first need to explain what $V_E, V_U, V_O$ are before getting into details. e.g., \alg keeps track of tree nodes using three sets: ...  Also explain the regions, e.g., \alg maintains the spatial information using a partition of state space denoted by $\mathcal{R}$.  Then, explain the lines of the algorithm below}.
\begin{algorithm}[t]
    \caption{\alg}
    \label{algo: alg}
    \SetKwInOut{Input}{Input}\SetKwInOut{Output}{Output}
    \Input{$x_{\init}, X_{\goal}, t_{max}$}
    \Output{Solution trajectory $\mathbf{x}$}
    
    \SetKwFunction{Propagate}{Propagate}
    \SetKwFunction{UpdateEstimates}{UpdateEstimates}
    \SetKwFunction{UpdateNodeSets}{UpdateNodeSets}

    $\mathcal{T} \gets$ Initialize tree with root node $x_{\init}$ \\
    $V_E \gets \{x_{\init}\}$, $V_{U}, V_{O} \gets \emptyset$ \\
    Initialize $\mathcal{R}$ with $P_{accept}(\mathcal{R}_i) = 1$ for each $\mathcal{R}_i \in \mathcal{R}$\\
    % For each region $\mathcal{R}_i$ in $\mathcal{R}$, set $\text{Accept}(\mathcal{R}_i) = 1$, $n_{invalid}(\mathcal{R}_i) = 0$, and $n_{valid}(\mathcal{R}_i) = 0$.

    % For each region $\mathcal{R}2_i$ in $\mathcal{R}2$, mark $\mathcal{R}2_i$ as unvisited.

    \While{$ElapsedTime < t_{max}$}{
        \Propagate{$V_E, V_{U}, \mathcal{R}, \lambda$} \\
        \UpdateEstimates{$\mathcal{R}$} \\
        $\mathbf{x} \gets$ \UpdateNodeSets{$V_{U}, V_E, V_{O}, \mathcal{T}, \mathcal{R}, X_{\goal}$} \\
        
        \lIf{$\mathbf{x} \neq null$}{
            \KwRet{$\mathbf{x}$}
        }
    }
    
    \KwRet{$null$}
\end{algorithm}

\subsubsection{Initialization}

In Alg.~\ref{algo: alg}, \texttt{\alg} takes as input the initial state $x_{\init}$, a goal region $X_{\goal}$, and a maximum execution time $t_{max}$. In Lines 1-2, a tree $\mathcal{T}$ is initialized with $x_{\init}$ at its root. 
Additionally, the set $V_E$ is initialized with the state $x_{\init}$ and the sets $V_{U}$ and $V_{O}$ are initialized as empty.
% The sets $V_{U}$, $V_{O}$, and $V_E$ partition all samples in \alg into three distinct categories, with each sample belonging to only one set. This partitioning is done using memory-efficient boolean masks, where a sample in \alg has a $true$ value for the corresponding index in one of $V_U$, $V_O$, or $V_E$. In implementation, \alg accesses these sets via a parallel efficient graph-search, as described in \cite{merrill2012scalable} and \cite{ichter2017group}.
% partitioning facilitates efficient memory management, in which each set can be represented by a boolean mask over $\mathcal{T}$ \qh{What does a boolean mask over $\mathcal{T}$ mean? It might be good to talk about the tree consisting of nodes and edges}.
In Line 3, the space decomposition \(\mathcal{R}\) is initialized in the state space, where the decomposition consists of non-overlapping regions, such that:
$$
X = \cup_{i=1}^n \mathcal{R}_i \;\;\; \text{ and } \;\;\;
\forall i\neq j \in \{1,\ldots,n\}, \;\;\; \text{ Int}(\mathcal{R}_i) \cap \text{Int}(\mathcal{R}_j) = \emptyset,
$$
% 
% Next, in Line 3 of Alg.~\ref{algo: alg}, the space decomposition \(\mathcal{R}\) and a fine overlaid decomposition \(\mathcal{R}2\) 
% \ml{consider a different name than R2, e.g., $\mathcal{R}^{\text{refined}}$. Also need to explain what this is perhaps after $R$ is defined.}
% are initialized. These decompositions can be defined in either the workspace \(\mathcal{W}\) or the state space \(X\)
% \ml{can it be in the workspace?  Does it not violate prob. completeness\footnote{THIS IS HOW TO MAKE A FOOTNOTE.}? Also, need to say something about which space the user should choose to decompose or refer the reader to a place they read more about it, e.g., A detailed discussion on this choice is provided in Sec. blah blah...}
% . In the state space, the decomposition consists of non-overlapping regions, such that:
% $$
% X = \bigcup_{i=1}^n \mathcal{R}_i \quad \text{and}
% \quad \forall i\neq j \in \{1,\ldots,n\}, \;\;\; \text{ Int}(\mathcal{R}_i) \cap \text{Int}(\mathcal{R}_j) = \emptyset,
% $$
% 
% \np{I removed the getRegion paragraph.}
% Each state \(x \in X \) is uniquely associated with a region \(\mathcal{R}_i\) using the function \(\texttt{getRegion}(x)\), which maps \(x\) to \(\mathcal{R}_i\) if and only if \(x \in \mathcal{R}_i\) \qh{where is this function? And this feels like too much implementation detail}. \np{I was following the lead of syclop paper. Do you think I should include the psuedocode for it or just somehow remove the references from the pseudocode?}
% 
where Int$(\mathcal{R}_i)$ is the interior of $\mathcal{R}_i$. Each $\mathcal{R}_i$ is then further partitioned to a set of finer regions. We denote the $k$-th sub-region of $\mathcal{R}_i$ by $\mathcal{R}^k_i$, i.e., $\mathcal{R}_i =  \cup_{k=1}^{n'} \mathcal{R}^k_i$. For each region $\mathcal{R}_i \in \mathcal{R}$, several metrics are calculated to assess the exploration progress of the tree. These metrics, adapted from \cite{plaku2010motion}, are designed to be effective in identifying promising regions for systems with complex dynamics and are well suited for parallelism. Specifically, $\alg$ updates the following metrics for each region, $\mathcal{R}_i$, after each iteration of parallel propagation to continually guide the search process:
\begin{itemize} 
    \item $Cov(\mathcal{R}_i)$: estimates the progress made by the tree planner in covering $\mathcal{R}_i$; 
    \item $FreeVol(\mathcal{R}_i)$: estimates the free volume of $\mathcal{R}_i$. 
\end{itemize}
The exact expressions for $Cov(\mathcal{R}_i)$ and $FreeVol(\mathcal{R}_i)$ are the same as in~\cite{plaku2010motion}, which we also show in Sec.~\ref{sec:node-selection}.

These metrics determine a score value, $Score(\mathcal{R}_i)$ and subsequently a probability \textbf{$P_{accept}(\mathcal{R}_i)$}, which aid \alg in adding favorable samples to $\mathcal{T}$ and assessing if an existing node should be extended. During initialization, all $P_{accept}(\mathcal{R}_i)$ values are set to 1.

% Based on these metrics, a sample acceptance probability, denoted as \textbf{$P_{accept}(\mathcal{R}_i)$}, is calculated to determine the likelihood of adding a newly generated sample within region $\mathcal{R}_i$ to the tree $\mathcal{T}$. This probability also determines whether a sample located in the set $V_{O}$ will be moved to the set $V_E$ for propagation in the next iteration of the parallel tree expansion.\qh{This paragraph needs polishing... We want to get this main idea but the details of how it's used is further down in the subroutines. }

\subsubsection{Node Extension}

\begin{algorithm}[t]
    \caption{Propagate}
    \label{algo: Propagate}
    \SetKwInOut{Input}{Input}\SetKwInOut{Output}{Output}
    \Input{$V_E, V_{U}, \mathcal{R}, \lambda$}
    \Output{Updated $V_{U}$}
    \ForEach{$x \in V_E$ }{
        \For{$i = 1, \dots, \lambda$}{
            % Propagate $x$, generating $x'$ \qh{this propagation needs more detail... sample controls, sample time duration}\;
            \rv{Randomly}
            sample $u$ and $dt$ \;
            $x' \gets \rv{\texttt{PropagateODE}}(x, u, dt)$ \;
            Map $x'$ to region $\mathcal{R}^k_i$ \;
            
            \If{the trajectory from $x$ to $x'$ is valid}{
                Increment $n_{valid}(\mathcal{R}_i)$ \;
               \If{$\mathcal{R}^k_i$ unvisited \textbf{or} with $P_{accept}(\mathcal{R}_i)$}{
                    Add $x'$ to $V_{U}$ \;
                }
            }
            \Else{
                Increment $n_{invalid}(\mathcal{R}_i)$ \;
            }
        }
    }
\end{algorithm}

After initialization, the main loop of the algorithm begins \rv{(Alg.~\ref{algo: alg}, Lines 4–8)}. In each iteration, the \texttt{Propagate} (Alg.~\ref{algo: Propagate}) function is called to propagate the set $V_E$ in parallel \rv{(Alg. \ref{algo: alg}, Line 5, Fig.~\ref{fig:iteration2})}. Each node $x \in V_E$ is expanded $\lambda \in \mathbb{N^+}$ times using $\lambda$ threads, where each thread handles one expansion of $x$ (Alg.~\ref{algo: Propagate}, Lines 1–2). 
\rv{For each thread, a control $u \in U$ and a time duration $dt \in (0, T_{prop}]$, where $T_{prop}$ is a user-defined constant that sets the maximum propagation time, are randomly
sampled, and the node's continuous state $x$ is propagated using dynamics in \eqref{eq:diffEq} to generate a new sample state $x'$ (Alg.~\ref{algo: Propagate}, Lines 3–4).}  
Next, the corresponding region of $x'$, $\mathcal{R}^k_i$, is calculated (Alg.~\ref{algo: Propagate}, Line 5). Then, in Line 6, the trajectory segment from $x$ to $x'$ is checked for validity, i.e., if it is in \rv{$X_\free$}. Throughout the trajectory segment, a user-defined collision check is performed (our implementation uses a coarse-phase bounding volume hierarchies method discussed in \cite{ichter2017group}).
% At each time step, a coarse-phase collision check is performed by creating a bounding box region $b \subseteq W$ around each $dt$ interval of the trajectory. Using bounding volume hierarchies, the segment is considered collision-free if the bounding box $b$ does not intersect any obstacles, i.e., if $b \cap \mathcal{O} = \emptyset$ \qh{is this collision checker necessary? or is it just the one we used?} \np{It can be any checker. this is just what we used.}. 
If the extended segment is valid, we increment the total number of valid samples in $\mathcal{R}_i$ \rv{(Alg.~\ref{algo: Propagate}, Line 7)}. Next, $x'$ is added to the set $V_U$ if its corresponding region $\mathcal{R}^k_i$ is unvisited; if $\mathcal{R}^k_i$ already contains a node, then $x'$ is added to $V_U$ with probability $P_{accept}(\mathcal{R}_i)$, which favors promising samples (Alg.~\ref{algo: Propagate}, Lines 8-9, Fig.~\ref{fig:iteration3}). Alternatively, if the trajectory segment is invalid, we increment the count of invalid samples in $\mathcal{R}_i$, as shown in Line 11. This information guides future propagation iterations away from regions that are frequently invalid, improving search efficiency. 

\begin{algorithm}[t]
    \caption{UpdateEstimates}
    \label{algo:UpdateEstimates}
    \SetKwInOut{Input}{Input}
    \SetKwInOut{Output}{Output}
    \Input{$\mathcal{R}$}
    \Output{Updated estimates for each region $\mathcal{R}_i$}
    
    \ForEach{$\mathcal{R}_i \in \mathcal{R}_{avail}$}{
        \texttt{UpdateFreeVol}($\mathcal{R}_i$)\;
        \texttt{UpdateCoverage}($\mathcal{R}_i$)\;
        \texttt{UpdateScore}($\mathcal{R}_i$)\;
    }

    \ForEach{$\mathcal{R}_i \in \mathcal{R}_{avail}$}{
        \texttt{UpdateAccept}($\mathcal{R}_i$)\;
    }
\end{algorithm}

\subsubsection{Node Selection}
\label{sec:node-selection}
After all samples in $V_E$ have been expanded, the \texttt{UpdateEstimates} (Alg.~\ref{algo:UpdateEstimates}) subroutine is called \rv{(Alg.~\ref{algo: alg}, Line 6)}. In this subroutine, metrics for each visited region (i.e., a region with a node $x \in \mathcal{T}$), denoted by $\mathcal{R}_{avail}$, are updated in parallel, with a thread handling a unique region $\mathcal{R}_i \in \mathcal{R}_{avail}$, calculating $Cov(\mathcal{R}_i)$ and $FreeVol(\mathcal{R}_i)$ (Alg.~\ref{algo:UpdateEstimates}, Lines 1-3). For each thread, $Cov(\mathcal{R}_i)$ is set to the number of visited sub-regions within $\mathcal{R}_i$, and $FreeVol(\mathcal{R}_i)$ is calculated as
\begin{equation}
    \label{eq:FreeVol}
    FreeVol(\mathcal{R}_i) = \frac{\big(\delta + n_{valid}(\mathcal{R}_i)\big)\cdot vol(\mathcal{R}_i)}{\delta + n_{valid}(\mathcal{R}_i) + n_{invalid}(\mathcal{R}_i)},
\end{equation}
where $\delta > 0$ is a small constant, and $vol(\mathcal{R}_i)$ represents the mapped workspace volume of the region $\mathcal{R}_i$. Subsequently, on Line 4 of Alg.~\ref{algo:UpdateEstimates}, each thread calculates its corresponding $Score(\mathcal{R}_i)$ value with
\small
\begin{equation}
    \label{eq:R_iw}
    Score(\mathcal{R}_i)= \frac{\text{FreeVol}^4(\mathcal{R}_i)}{(1 + \text{Cov}(\mathcal{R}_i))(1 + (n_{valid}(\mathcal{R}_i) + n_{invalid}(\mathcal{R}_i))^2)},
\end{equation}
\normalsize
which prioritizes regions that are less visited and have a high free volume and low coverage.

% Our method for calculating  $Cov(\mathcal{R}_i)$, $FreeVol(\mathcal{R}_i)$ and $Score(\mathcal{R}_i)$ are particularly well suited for parallelism because each thread follows the same sequence of operations (\emph{even workloads across threads}) and threads act independently (\emph{independence between threads}).

Once all score values have been updated, each visited region's  $P_{accept}(\mathcal{R}_i)$ probability is refined (Alg.~\ref{algo:UpdateEstimates}, Lines 5-6, Fig.~\ref{fig:iteration3}). This process, again, is done in parallel with a thread being designated to a unique $\mathcal{R}_i \in \mathcal{R}_{avail}$ and $P_{accept}(\mathcal{R}_i)$ being set by
\begin{equation}
    \label{eq:Accept}
    P_{accept}(\mathcal{R}_i) = \min \left\{1, \;\;\;
    \frac{Score(\mathcal{R}_i)}{\sum_{\mathcal{R}_j \in \mathcal{R}_{\text{avail}}} Score(\mathcal{R}_j)} + \epsilon \right\},
\end{equation}
where $0<\epsilon \ll 1$ is a constant and $\mathcal{R}_{avail} \in \mathcal{R}$ represents the set of regions that contain a node in $\mathcal{T}$. We note that the expressions for the metrics in \eqref{eq:FreeVol}-\eqref{eq:Accept} are taken from~\cite{plaku2010motion}.
% Since this calculation depends on the sum of scores from each available region, \alg shares information via the small amount of on-chip low-latency memory and then performs a hierarchical reduction, as described in \cite{kirk2016programming}.

% \qh{Do we want to mention somewhere that this method minimizes uneven workloads across device?}After all samples in $V_E$ have been expanded, the values of \textbf{$Cov(\mathcal{R}_i)$}, \textbf{$FreeVol(\mathcal{R}_i)$}, and \textbf{$P_{accept}(\mathcal{R}_i)$} are updated in parallel \qh{what do you mean updated in parallel? Do multiple threads/workers never have to handle different $R_i$?} for each region in $\mathcal{R}$, in line 8 of \ref{algo: alg} and detailed in subroutine \ref{algo:UpdateEstimates}, \texttt{UpdateEstimates}. On line 2 of the subroutine, each thread independently updates its free volume \qh{how can each thread independently update only its own free volume? How about information from other threads?} using the following equation:

% \begin{equation}
% \label{eq:FreeVol}
% \text{FreeVol}(\mathcal{R}_i) = \frac{\epsilon + n_{valid}(\mathcal{R}_i)}{\epsilon + n_{valid}(\mathcal{R}_i) + n_{invalid}(\mathcal{R}_i)} \times \text{vol}(\mathcal{R}_i)
% \end{equation}

Once the $P_{accept}(\mathcal{R}_i)$ probabilities have been updated for all available regions, the \texttt{UpdateNodeSets} subroutine is called \rv{(Alg.~\ref{algo: alg}, Line $7$, Fig.~\ref{fig:iteration4})}. In Lines $1-3$ of Alg.~\ref{algo:UpdateNodeSets}, we remove samples from the expansion set $V_E$ randomly with probability $1 - P_{accept}(\mathcal{R}_i)$, ensuring that promising samples remain in $V_E$. Then, we move the newly generated samples, $V_U$, to $\mathcal{T}$ and add them to the expansion set $V_E$. If any newly generated nodes satisfy goal criteria, we return the valid trajectory $\mathbf{x}$ (Alg.~\ref{algo:UpdateNodeSets}, Lines 4-6). Finally, we move inactive samples in $V_{O}$ to the expansion set if deemed promising with the updated search information (Alg.~\ref{algo:UpdateNodeSets}, Lines 7-9).

\alg repeats the main loop of \texttt{Propagate}, \texttt{UpdateEstimates} and \texttt{UpdateNodeSets} until a solution trajectory $\mathbf{x}$ that solves Problem~\ref{problem} is returned, or a user-defined time limit $t_{max}$ is surpassed.

\begin{algorithm}[t]
    \caption{UpdateNodeSets}
    \label{algo:UpdateNodeSets}
    \SetKwInOut{Input}{Input}
    \SetKwInOut{Output}{Output}
    \Input{$V_{U}, V_E, V_{O}, \mathcal{T}, \mathcal{R}, X_{\goal}$}
    \Output{Trajectory if a goal is found, otherwise $null$}
    
    \ForEach{$x \in V_E$ }{
        Map $x$ to $\mathcal{R}_i$\;
        Move $x$ to $V_O$ with probability $1-P_{accept}(\mathcal{R}_i)$;
        % \If{$uniform(0,1) < 1 - P_{accept}(\mathcal{R}_i)$}{
        %     Remove $x$ from $V_E$ and place in $V_{O}$\;
        % }
    }

    \ForEach{$x \in V_{U}$ }{ 
        Move $x$ from $V_{U}$ to $V_E$ and $\mathcal{T}$\;
        \lIf{$x \in X_{\goal}$}{
            \Return Trajectory $x_{\init}$ to $x$
        }
    }
    \ForEach{$x \in V_{O}$ }{
        Map $x$ to $\mathcal{R}_i$\;
        Move $x$ to $V_E$ with probability $P_{accept}(\mathcal{R}_i)$;
    }

    \Return $null$
\end{algorithm}

\section{\rv{Tuning and Performance Discussion}}
    \label{sec:GPU_implementation}
    % These GPUs operate using a three-level execution hierarchy. At the lowest level are \emph{warps}, which consist of groups of 32 threads that execute the same instruction simultaneously on different data. If threads within a warp follow divergent execution paths, such as those caused by conditional statements, some threads become idle, reducing efficiency. The next level of the CUDA hierarchy consists of blocks, which are groups of threads that can share data through a small amount of on-chip shared memory. This memory enables fast communication and synchronization among threads within the same block. However, threads from different blocks cannot share memory directly. Finally, at the highest level, blocks are organized into grids, where they are dispatched to the device for execution.

% \rv{In this section, we discuss how \alg can be tuned to match a problem's difficulty, present the properties of \alg that enable efficient parallelism, and analyze its scalability.}

In this section, we discuss how \alg can be tuned to match a problem's difficulty, and present the properties of \alg that enable efficient parallelism.

\subsection{Tuning Parameter}

In practice, due to the limitations of device memory and the high cost of vector resizing operations, we predefine the maximum size of the tree rather than constraining the runtime, similar to the approach taken in PRM. This introduces a hyperparameter for \alg, denoted as $t_e$,  which we refer to as the expected tree size. Specifically, $t_e$ corresponds to the maximum number of nodes in \alg and should be tuned according to the difficulty of the problem at hand.

Varying, $t_e$ has two main effects on the performance of \alg. 
Firstly, an increase in $t_e$ increases the branching factor $\lambda$ which is updated for each iteration of parallel propagation and is set according to
\begin{equation}
\label{eq:BF}
    \lambda = \min \left\{ \lambda_{\text{max}}, \; \left\lfloor \frac{t_e - |\mathcal{T}|}{|V_E|} \right\rfloor \right\},
    % \lambda = \min \left\{ \lambda_{\text{max}}, \; \left\lfloor (t_e - |\mathcal{T}|)/|V_E| \right\rfloor \right\},
\end{equation}
where $\lambda_{max}$ is the user-defined maximum branching factor and $|\mathcal{T}|$ and $|V_E|$ are the numbers of nodes in $\mathcal{T}$ and $V_E$, respectively. Eq.~\eqref{eq:BF} ensures that in the early iterations, when $|\mathcal{T}|$ is much smaller than $t_e$, a larger $\lambda$ is used. This approach effectively uses available throughput and accelerates the initial stages of tree propagation. \rv{As $|\mathcal{T}|$ approaches $t_e$ and $|V_E|$ grows large, a smaller $\lambda$ is used to stabilize the growth rate of $\mathcal{T}$.}

% \rv{A larger $\lambda$ also leads to more trajectory segments originating from the same state, resulting in a denser tree structure. As $|\mathcal{T}|$ approaches $t_e$ and $|V_E|$ becomes large, a smaller $\lambda$ is used to prevent device serialization, which occurs when the parallel subroutine requires more threads than the device can support. Device serialization is undesirable, as it forces the subroutine to execute in multiple waves of computation, with intermittent idle periods between waves. Consequently, a smaller $\lambda$ results in fewer trajectory segments stemming from the same state, leading to less densely located nodes and increased tree sparsity.}

% amount of average concurrent threads during node expansion
% \ml{not clear how this is so.  }
% by making the branching factor $\lambda$ larger
% \ml{how's $\lambda$ realted to $t_e$?  the connection is unclear.}
% . The branching factor, is updated for each iteration of parallel propagation and is set by the following equation:
% % 
% \begin{equation}
% \label{eq:BF}
%     \lambda = \min \left\{ \lambda_{\text{max}}, \; \left\lfloor \frac{t_e - |\mathcal{T}|}{|V_E|} \right\rfloor \right\},
% \end{equation}
% % 

Secondly, $t_e$ affects the number of nodes that can be included in $\mathcal{T}$. As the problem difficulty increases, such as with a system's state dimension, a larger $t_e$ is recommended to sufficiently explore the space. \alg's search characteristics make finding a suitable $t_e$ relatively straightforward for a given system, as the tree expands in all accessible free space areas. We demonstrate this in Sec.~\ref{sec:experiments}, where systems with the same state dimension utilize a constant $t_e$ value across all testing environments. In Sec.~\ref{sec:experiments}, we also show that as $t_e$ increases, the success rate of \alg converges to 100\%.

\begin{remark}
    The maximum tree size can be made adaptive by increasing $t_e$ by a constant multiple if a solution is not found after the tree size nears its threshold. However, this introduces vector resizing operations on high latency device memory, slowing down the search.
\end{remark}

\subsection{Decomposition Tuning}

% \np{Should we mention our experiments?}
\rv{The search efficiency of \alg is dependent on the choice of space decomposition. An improper decomposition, e.g., one that is too coarse or too fine, can lead to inefficiencies. A decomposition that is too fine may cause slower updates due to the large number of regions. On the other hand, a decomposition that is too coarse may lead to poor approximation of promising space regions. For instance, if a region frequently generates invalid trajectories but contains critical space for finding a solution, \alg may experience inefficient search. 
A method to mitigate this is to use a free-space-obeying decomposition, as proposed in \cite{plaku2010motion}.
Nonetheless, \alg remains probabilistically complete with any valid decomposition.
}

% \begin{remark}
%     \rv{Free-space-obeying decompositions, e.g., the one proposed in \cite{plaku2010motion}, can be used to mitigate search inefficiencies.}
% \end{remark}

 % For instance, if $\mathcal{R}$ includes a region that frequently generates invalid trajectories but contains critical space for finding a solution, such as a narrow passage with a coarse decomposition, .

% \rv{\alg's space decomposition, $\mathcal{R}$, requires tuning to improve search efficiency. Since \alg favors the exploration of unexplored and safe regions, an improper decomposition—e.g., one that is too coarse—can lead to a poor approximation of promising space regions. For instance, if $\mathcal{R}$ includes a region that frequently generates invalid trajectories but contains critical space for finding a solution, such as a narrow passage with a coarse decomposition, \alg may experience search inefficiencies.}

% \rv{In these scenarios, refining \alg's space decomposition introduces minimal computational overhead while significantly improving search efficiency, as demonstrated in the narrow passage experiment in \ref{sec:experiments}.}

\subsection{Efficient Application to Highly Parallel Devices}

\alg's propagation subroutine implementation is well-suited for parallelism due to three main factors. \rv{First, we utilize \emph{low-latency memory} by distributing commonly used data to groups of nearby threads. Specifically, the state $x \in V_E$ is shared via on-chip memory to all $\lambda$ threads assigned to expand the node.} \rv{Second, we \emph{balance workloads} across threads by having each thread create a single trajectory segment, minimizing thread divergence—i.e., variations in execution paths that force threads into serial execution.} \rv{Third, the acceptance of new samples and their addition to $V_U$ is achieved with \emph{thread independence}, using \eqref{eq:Accept} and unique thread identifiers.} 
% This approach ensures independence and avoids memory write conflicts.

\alg maintains its space decomposition in a parallel-friendly manner through the metrics \eqref{eq:FreeVol}-\eqref{eq:R_iw} that can be calculated independently of other regions and by ensuring that each region's calculations have an equal number of operations. Further, we avoid the need for serial data structures when choosing expansion nodes by adding samples to $V_E$ independently via the acceptance probability \eqref{eq:Accept}.

Furthermore, the organization of samples into three disjoint sets, $V_{U}, V_{O},$ and $V_E$, enables a straightforward and memory-efficient representation. In \alg, we do this via a boolean-mask representation that is thoroughly discussed in \cite{merrill2012scalable}.

\rv{
Lastly, \alg reduces latency between its subroutines by pre-allocating a large memory chunk on the GPU to accommodate $t_e$ nodes. This allows \alg to construct its tree directly on the GPU, avoiding the transfer of large data structures between devices.
}

\section{Analysis}
    \label{sec:Analayis}
    % \np{I have swapped the analysis section and the implementation discussion, as recommended by Reviewer 7. However, the analysis section discusses scalability, which refers to the vector resizing operations explained in the implementation discussion.}

% \rv{Here, we show that \alg is probabilistically complete for Problem~\ref{problem}. We begin by providing a definition of probabilistic completeness for algorithms that solve Problem~\ref{problem}.}

Here, we show that \alg is probabilistically complete for Problem~\ref{problem} and analyze its scalability.

\subsection{Probabilistic Completeness}

We begin with a definition of probabilistic completeness for algorithms that solve Problem \ref{problem}.

% In this section, we show that \alg is probabilistically complete for Problem~\ref{problem}. We begin with a definition of probabilistic completeness for algorithms that solve Problem \ref{problem}.

\begin{definition}[Probabilistic Completeness]
    \label{def: prob complete}
    A sampling-based algorithm is \emph{probabilistically complete} if the probability that the algorithm fails to return a solution, given one exists, approaches zero as the number of samples approaches infinity.    
\end{definition}

We first show that,
in \alg,
every tree node has a non-zero probability of being extended.

\begin{lemma}
    \label{lemma: nonzero prob}
    Let $x \in \mathcal{T}$ be a node in the tree. The probability that $x$ is selected for extension is lower bounded by $\epsilon \in (0,1)$.
\end{lemma}
\begin{proof}
    The probability of extending a node $x \in \mathcal{T}$ in \alg is calculated using \eqref{eq:Accept} for its corresponding region $\mathcal{R}_i \in \mathcal{R}$. We demonstrate that every region $\mathcal{R}_i$ falls into one of three cases, and in each case, \eqref{eq:Accept} is lower bounded by $\epsilon$. 

    \emph{Case 1}: a region $\mathcal{R}_i$ is entirely in invalid space, i.e., $\mathcal{R}_i \cap X_{free} = \emptyset$, meaning $\mathcal{R}_i$ cannot contain a tree node and $\mathcal{R}_i \notin \mathcal{R}_{avail}$. Thus, \eqref{eq:Accept} is set to $1 > \epsilon$, as per Line 3 of Alg.~\ref{algo: alg}.

    \emph{Case 2}: a region $\mathcal{R}_i$ is entirely within free space, i.e., $\mathcal{R}_i \subseteq X_{\free}$. If $\mathcal{T}$ does not contain a node in $\mathcal{R}_i$, then $\mathcal{R}_i \notin \mathcal{R}_{avail}$, and \eqref{eq:Accept} is set to $1 > \epsilon$, as by Line 3 of Alg.~\ref{algo: alg}. However, if $\mathcal{T}$ does contain a node in $\mathcal{R}_i$, we examine the worst-case scenario that produces the lowest probability from \eqref{eq:Accept}. In this scenario, the number of nodes in $\mathcal{R}_i$ approaches infinity, driving the score of $\mathcal{R}_i$ to $0$ by (\ref{eq:R_iw}). Additionally, in a worst-case scenario, the score of all other available regions approaches infinity. In this case, \eqref{eq:Accept} trivially approaches $\epsilon$.

    \emph{Case 3}: $\mathcal{R}_i$ is partially in $X_{\free}$, i.e., $\mathcal{R}_i \cap X_\free \neq \emptyset$ and $\mathcal{R}_i \cap X_\free \neq \mathcal{R}_i$. In this case, similar to above, if $\mathcal{R}_i \notin \mathcal{R}_{avail}$, \eqref{eq:Accept} is set to $1$ by Line 3 of Alg.~\ref{algo: alg}. If $\mathcal{R}_i \in \mathcal{R}_{avail}$, as in \textit{Case 2}, \eqref{eq:Accept} approaches $\epsilon$ in the worst-case, when the number of nodes in $\mathcal{R}_i$ approaches infinity.
\end{proof}

Finally, we can state our main analysis result.

\begin{theorem}
    \alg is probabilistically complete.
\end{theorem}
\begin{proof}[Proof Sketch]
    % Our proof adapts the result of \cite{kleinbort2018probabilistic}, proven for the case of kinodynamic RRT, to \alg. Specifically, the conditions for \cite[Lemma 2–3]{kleinbort2018probabilistic} also hold for Problem~\ref{problem} in this paper. Additionally, Lemma~\ref{lemma: nonzero prob} provides a lower bound $\epsilon > 0$ on the probability of extending from an arbitrary tree node. 
    % Using \cite[Lemma 2-3]{kleinbort2018probabilistic} and Lemma~\ref{lemma: nonzero prob}, we can use the same logical structure of the proof of \cite[Theorem 2]{kleinbort2018probabilistic}. Therefore, we have that, as number of samples approach infinity, \alg asymptotically, almost-surely finds a trajectory from initial state $x_\init$ to a goal region $X_{\goal}$ that is a solution to Problem~\ref{problem}.
    Our proof is an adaptation of the result from \cite{kleinbort2018probabilistic}, originally established for kinodynamic RRT, for \alg. Specifically, the conditions laid out in \cite[Lemma 2–3]{kleinbort2018probabilistic} hold under the assumptions of Problem~\ref{problem} in this paper. Furthermore, Lemma~\ref{lemma: nonzero prob} provides a lower bound $\epsilon > 0$ on the probability of extending from an arbitrary tree node. 
    By combining \cite[Lemma 2-3]{kleinbort2018probabilistic} with Lemma~\ref{lemma: nonzero prob}, we can follow the same logical structure as the proof of \cite[Theorem 2]{kleinbort2018probabilistic}. Thus, as the number of samples approaches infinity, \alg asymptotically almost-surely finds a valid trajectory from $x_\init$ to $X_{\goal}$ that solves Problem~\ref{problem}.
\end{proof}

\subsection{Scalability}

Here, we discuss the scalability of \alg.
Firstly, \alg's scalability increases as the number of cores in the parallel device increases. Since \alg supports adaptive tuning of the branching factor $\lambda$ through varying $t_e$, an increased number of cores can be leveraged by increasing $t_e$ and setting a higher $\lambda_{max}$. This results in computation time improvements as the number of cores increase, allowing \alg to scale as parallel hardware computation power improves.

Secondly, as the problem becomes more complex, i.e., requiring a larger tree (more nodes) to find a solution, traditional tree-based SBMPs suffer.  That is, they slow down significantly as the number of nodes increases due to the sequential nature of those algorithms. 
% , which struggles with handling the growing number of nodes in the tree structure. 
However, \alg does not struggle with increasing number of nodes, unless the tree size nears its threshold and a resizing operation is needed.
% when the tree reaches its maximum size.
% 
% \qh{Is this true? If $t_e$ is huge, it may cause $\lambda$ to be too large for the parallel computation? Maybe something about as long as $\lambda$ is low enough?} \np{yes, $\lambda$ is limited by user-defined parameter $\lambda_{max}$}.
% 
This property makes \alg particularly more advantageous for planning for systems with large dimensional state spaces
% . That is, existing kinodynamic SBMP suffer in such problems since 
% large trees are often required to sufficiently search the space.  For \alg, a large tree size does not cause a slowdown. 
% \ml{Growing tree size does however affect its performance negatively if $t_e$, and thus maximum tree size, must be modified on the fly.}
since they often require large trees to sufficiently search the space.

\section{Experiments}
    \label{sec:experiments}
    We demonstrate the performance of \alg in planning for various dynamical systems across three 3D environments shown in Fig.~\ref{fig:environments}. 
% We apply \alg to two 6-dimensional state space systems: 
The considered systems are:
(i) 6D double integrator, (ii) 6D Dubins airplane~\cite{chitsaz2007time}, and (iii) 12D highly nonlinear quadcopter~\cite{etkin1995dynamics}.  

\begin{figure}[b]
    \centering
     \begin{subfigure}[b]{0.32\columnwidth}
        \centering
        \includegraphics[width=\textwidth]{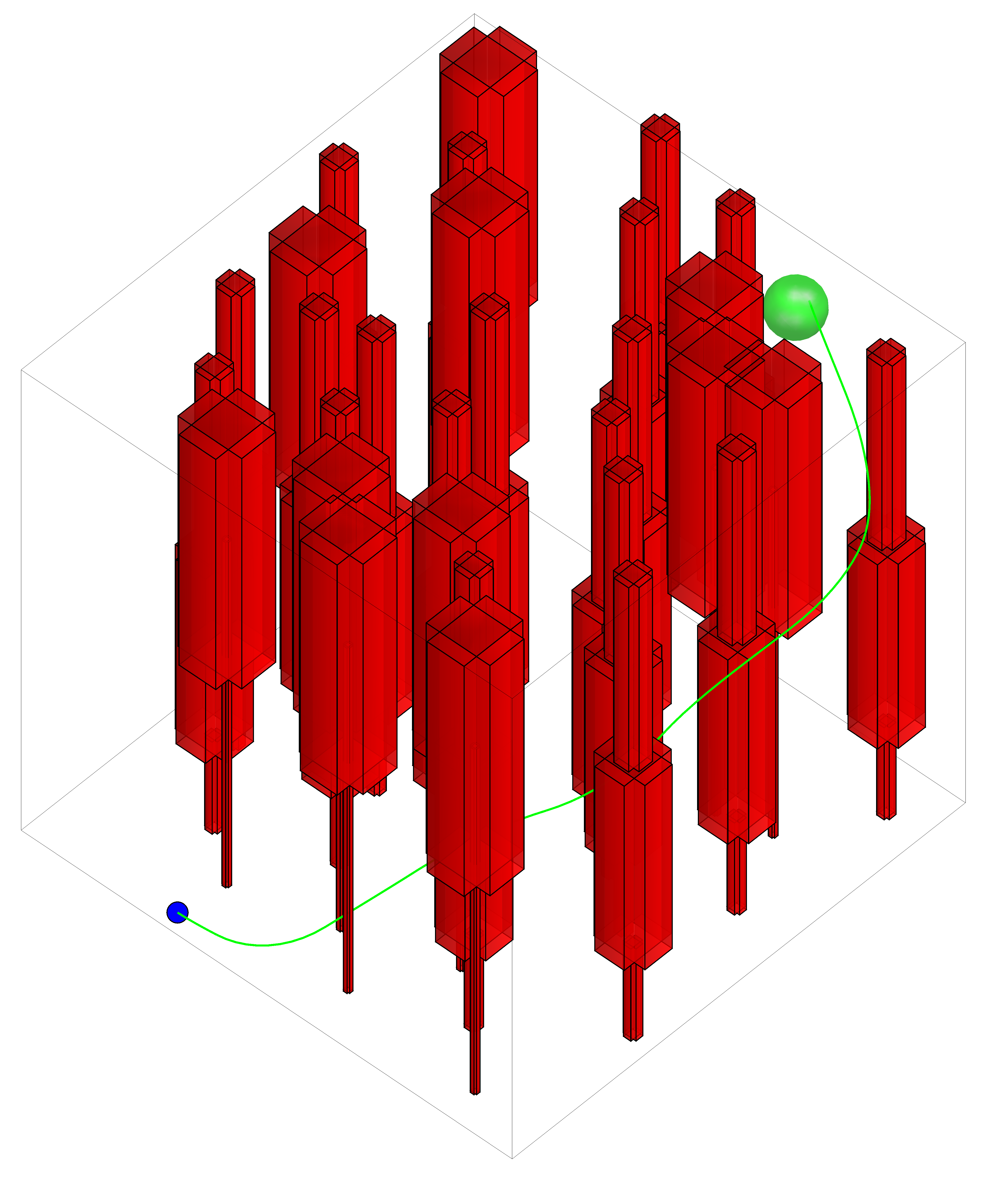}
        \caption{Forest}
        \label{fig:trees}
    \end{subfigure}
    % \hfill
    \begin{subfigure}[b]{0.30\columnwidth}
        \centering
        \includegraphics[width=\textwidth]{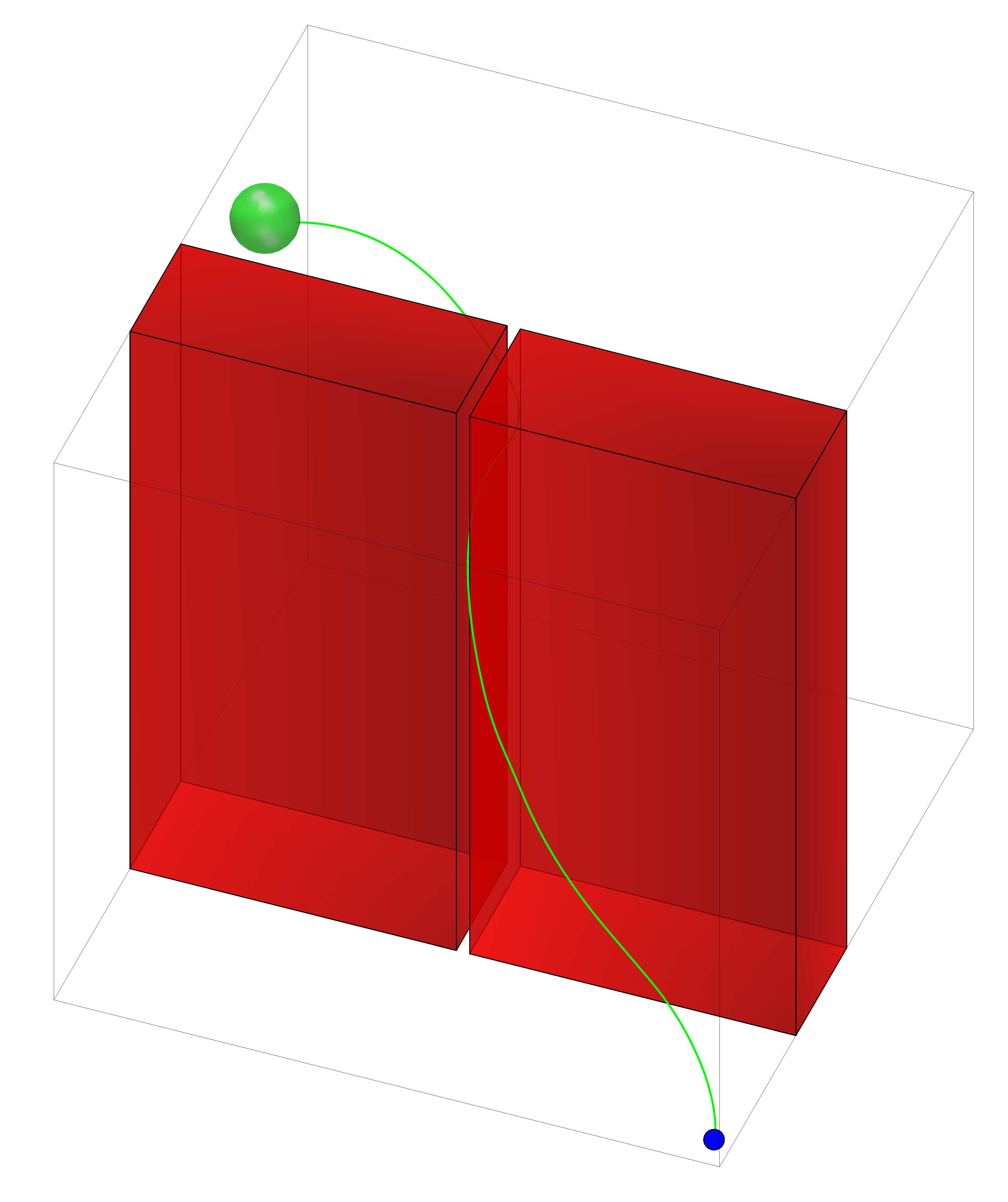}
        \caption{Narrow Passage}
        \label{fig:narrowPassage}
    \end{subfigure}
     % \hfill
    \begin{subfigure}[b]{0.34\columnwidth}
        \centering
        \includegraphics[width=\textwidth]{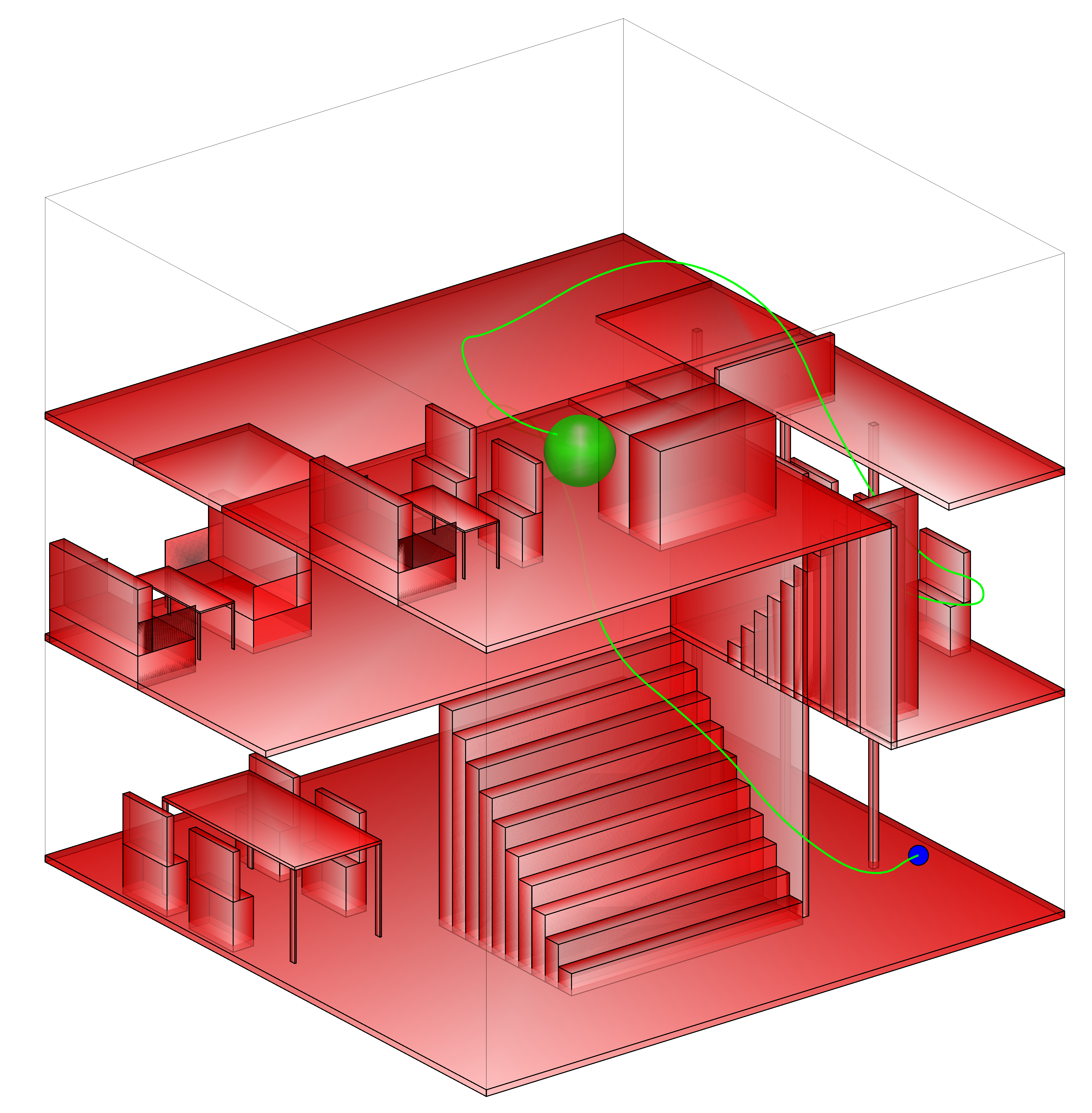}
        \caption{Building}
        \label{fig:house}
    \end{subfigure}
    \caption{Environments used throughout the experiments with the solution trajectory produced by \alg. Initial position and goal region are shown by blue and green spheres, respectively.
    % (a) A forest environment with many trees of varying sizes, taken from \cite{ichter2017group}. (b) A long, narrow passage. (c) A multi-story building \cite{ichter2017group}.
    Environments (a) and (c) are taken from \cite{ichter2017group}.  
    }
    \label{fig:environments}
\end{figure}

\begin{table*}[t]
    \centering
    % \caption{Benchmark results
    % % . All planners were given a maximum runtime of \maxT seconds, and each metric represents the mean of \itr executions. 
    % over 50 trials with maximum planning time of \maxT seconds.
    % % The planners employed identical propagation and collision-checking logic. Additionally, planners utilizing a decomposition method applied the same decomposition strategy.
    % }
    \caption{\rv{Benchmark results over 50 trials with a 60-second maximum planning time. CPU-based algorithms used coarse-grained parallelization, growing multiple trees in parallel and utilizing all available cores, denoted by ``Par'' before the algorithm name.} 
    }
    \label{table:results}
    \resizebox{\textwidth}{!}{
    \begin{tabular}{l l |  rrr | rrr  | rrr}
        % \hline
        \toprule
        \multirow{2}{*}{Algorithm} & \multirow{2}{*}{Device} & \multicolumn{3}{c}{\underline{\hspace{12mm}Environment \subref{fig:trees} \hspace{12mm}}} & \multicolumn{3}{c}{\underline{\hspace{12mm}Environment \subref{fig:narrowPassage}\hspace{12mm}}} & \multicolumn{3}{c}{\underline{\hspace{12mm}Environment \subref{fig:house}\hspace{12mm}}} \\ 
        % \cline{3-11} 
                                   &                         & Time (ms) & $t_{alg}/t_{\alg}$ & Succ \% & Time (ms) & $t_{alg}/t_{\alg}$ & Succ \% & Time (ms) & $t_{alg}/t_{\alg}$ & Succ \% \\ 
        \hline
        % \midrule
        \multicolumn{11}{c}{6D Double Integrator} \\ 
        \hline
        % \midrule
        \rv{Par} RRT                       & CPU     & 157.0 & 42.0 & 100.0 & 598.1  & 184.1 & 100.0 & 993.7  & 183.4 & 100.0 \\
        \rv{Par} EST           & CPU     & 416.1 & 111.4 & 100.0 & 1822.5 & 561 & 100.0 & 5122.8 & 945.5 & 100.0 \\
        \rv{Par} PDST              & CPU     & 486.0 & 130.1 & 100.0 & 1144.7  & 352.4 & 100.0 & 2094.5 & 386.6 & 100.0 \\
        \rv{Par} SyCLoP             & CPU     & 216.6 & 58.0 & 100.0 & 168.1  & 51.8 & 100.0 & 1183.3 & 218.4 & 100.0 \\
        \alg & Embd. GPU     & 62.3  & 16.7 & 100.0 & 29.1  & 9.0 & 100.0 & 78.6  & 14.5 & 100.0 \\ 
        \alg & GPU     & \textbf{3.7} & 1.0 & 100.0 & \textbf{3.3} & 1.0 & 100.0 & \textbf{5.4} & 1.0 & 100.0 \\ 
        \hline
        % \midrule
        \multicolumn{11}{ c }{Dubins Airplane} \\ 
        \hline
        % \midrule
        \rv{Par} RRT                       & CPU     & 632.5  & 165.0 & 100.0 & 4973.9  & 1389.9 & 100.0 & 22698.7  & 3207.4 & 98.0 \\
        \rv{Par} EST           & CPU     & 275.7 & 71.9 & 100.0 & 3443.6 & 962.3 & 100.0 & 9108.5 & 1287.1 & 100.0 \\
        \rv{Par} PDST              & CPU     & 515.6 & 134.5 & 100.0 & 12410.3   & 3467.9 & 98.0 & 18570.2 & 2624.0 & 100.0 \\
        \rv{Par} SyCLoP              & CPU     & 234.4 & 61.2 & 100.0 & 936.2  & 261.6 & 100.0 & 25544.7 & 3610.0 & 100.0 \\
        \alg & Embd. GPU     & 67.5  & 17.6 & 100.0 & 43.1 & 12.0 & 100.0 & 110.5 & 15.6 & 100.0 \\ 
        \alg & GPU     & \textbf{3.8} & 1.0 & 100.0 & \textbf{3.6} & 1.0 & 100.0 & \textbf{7.1} & 1.0 & 100.0 \\ 
        \hline
        % \midrule
        \multicolumn{11}{ c }{12D Non Linear Quadcopter} \\ 
        \hline
        % \midrule
        \rv{Par} RRT                       & CPU     & 40694.7  & 2262.5 & 72.0 & 91023.1 &  5306.9 & 12.0 & 93144.7  &  3873.2 & 8.0 \\
        \rv{Par} EST           & CPU     & 10034.4 & 557.9 & 100.0 & 35435.9 & 2066 & 90.0 & 41381.0 & 1720.8 & 84.0 \\
        \rv{Par} PDST              & CPU     & 23726.9 & 1319.1 & 92.0 & 51203.4   & 2985.3 & 70.0 & 53169.8 & 2211 & 68.0 \\
        \rv{Par} SyCLoP              & CPU     & 3384.2 & 188.2 & 100.0 & 10181.6   & 593.6 & 100.0 & 86558.3 & 3599.4 & 16.0 \\
        \alg & Embd. GPU     & 797.0 & 44.3 & 100.0 & 681.8 & 39.8 & 100.0 & 935.7 & 38.9 & 100.0 \\ 
        \alg & GPU     & \textbf{18.0} & 1.0 & 100.0 & \textbf{17.2} & 1.0 & 100.0 & \textbf{24.1} & 1.0 & 100.0 \\ 
        % \hline 
        \bottomrule
    \end{tabular}
    }
\end{table*}

For each pair of dynamical system and environment, we benchmark the performance and scalability of \alg against four traditional SBMPs: RRT~\cite{lavalle2001randomized}, EST~\cite{hsu1997path}, PDST~\cite{ladd2005fast} and SyCLoP~\cite{plaku2010motion}. 
To ensure fairness, for these comparison planners, we used a coarse-grained CPU parallelization method where multiple trees grow in parallel as suggested by~\cite{746692}.
Each tree is managed by a separate thread, and the planner returns the first solution found.

We implemented \alg in CUDA C and performed benchmarks on two GPUs with different capabilities. We used an NVIDIA RTX 4090 as a baseline, which has 16,384 CUDA cores and 24 GB of RAM. Further, to test the efficiency of \alg on an embedded GPU, we ran benchmarks on an NVIDIA Jetson Orin Nano, which has 1,024 cores and 8 GB of RAM. The comparison algorithms, are implemented in C++ using \textit{OMPL} \cite{sucan2012the-open-motion-planning-library} and executed on an Intel Core i9-14900K CPU with 24 cores, a base clock speed of 4.4 GHz, and 128 GB of RAM.
Our implementation of \alg is publicly available: \url{https://github.com/aria-systems-group/Kino-PAX}.

% We followed the standard setup configurations for all comparison algorithms as recommended by the \textit{OMPL} documentation. All algorithms, including \alg, were configured to use the same methods for state propagation, state validity checking, and space decomposition. For all experiments, a grid-based decomposition is used, with the dimensionality of the grid equal to the system's state-space dimension. For \alg's hyperparameter a $t_e$ value of $2 \times 10^5$ is used for all 6-dimensional dynamical systems, and for the 12-dimensional system, $t_e$ is set to $4 \times 10^5$. For each combination of algorithm, dynamic model, and workspace, we performed \itr queries, each with a maximum runtime of \maxT seconds.
We followed the standard setup configurations for all comparison algorithms as recommended by the \textit{OMPL} documentation. All algorithms, including \alg, were configured to use the same methods for state propagation, state validity checking and space decomposition. For all experiments, a grid-based decomposition was used with the dimensionality of the grid equal to the system's state-space dimension. For \alg's hyperparameters, $\lambda_{max}$ was set to $32$ and \(t_e\) was set to \(2 \times 10^5\) for all 6D systems and \(4 \times 10^5\) for the 12D system. For each combination of algorithm, dynamic model, and workspace, we performed \itr queries, each with a maximum runtime of \maxT seconds.

\subsection{Benchmark Results}

% \np{Mean runtime for Dubins Airplane on Environment 2, with workspace decomposition is 9.19ms}

Table \ref{table:results} shows the mean runtime, the speed ratio relative to the desktop GPU implementation of \alg, and the success rate within the allotted planning time for each combination of algorithm, environment, and dynamics.

For both 6D systems, \alg finds a solution trajectory in less than $8$ ms across all testing environments. For the 6D Double Integrator, \alg is on average $85\times$, $287\times$, $433 \times$ faster in Environments  \subref{fig:trees}, \subref{fig:narrowPassage}, \subref{fig:house}, respectively, 
% and  faster in Environment 1, $287\times$ faster in Environment 2, and 433 times faster in Environment 3 
compared to the baseline algorithms. 
For the Dubins Airplane system, the performance gap of \alg widens further. For instance, in Environment \subref{fig:house}, \alg experiences a slowdown of less than $1.3 \times$ ($\sim$2 ms) compared to RRT's $22 \times$ ($\sim$20,000 ms), EST's $1.8 \times$ ($\sim$4,000 ms), PDST's $8.9 \times$ ($\sim$16,000 ms), and SyCLoP's $22 \times$ ($\sim$24,000 ms). Additionally, the embedded GPU implementation of \alg outperforms all serial baseline methods, finding valid trajectories for all 6D problems in under $115$ ms. When dealing with the more challenging 12D nonlinear quadcopter problem, \alg finds solutions in less than $25$ ms across all environments. On average, this marks an improvement of \emph{three orders of magnitude} over all reference serial solutions. In the most challenging environment (Environment \subref{fig:house}), the best-performing baseline algorithm (EST) is $1720 \times$ slower than the desktop implementation of \alg and $44 \times$ slower than the embedded GPU implementation.

As evident from the results, \alg outperforms baseline algorithms more significantly as the problem becomes more challenging; in other words, the performance gap significantly widens in favor of \alg. This is due to two main factors. First, as the dimensionality of the search space increases, exponentially more trajectory segments are required to find a valid solution. This suits \alg particularly well, as it is designed to propagate a massive number of nodes efficiently in parallel. Second, as the problem difficulty increases, the efficiency of \alg becomes more prominent. This is because unlike traditional tree-based SBMPs that slow down as the number of samples increases, \alg does not suffer as much with the size of the tree.

\subsection{Effects Of Tuning Parameter $t_e$}

To support the point made in Sec.~\ref{sec:GPU_implementation} that \alg's hyperparameter $t_e$ is easy to tune and that \alg remains efficient across a wide range of values, we present a numerical experiment showing that as \alg is provided with a sufficiently large $t_e$, its failure rate converges to zero. We also examine the impact on runtime as $t_e$ increases. 
% Figures \ref{fig:NumFailuresVsMemorySize} and \ref{fig:RuntimeVsMemorySize} 
Fig.~\ref{fig:iteration_grid}
presents the results for planning with the 12D nonlinear quadcopter system in Environment \subref{fig:trees}.

% To complement the proof in Section \ref{sec:Analayis} that \alg is probabilistically complete, we present a numerical experiment demonstrating that as \alg is provided with a sufficiently large expected tree size, \(t_e\), its failure rate converges to zero. Additionally, we examine the effects on runtime as \(t_e\) increases. Figures \ref{fig:NumFailuresVsMemorySize} and \ref{fig:RuntimeVsMemorySize} illustrate the results for planning with the 12D Nonlinear Quadcopter system in Environment 1. 

As shown in Fig.~\ref{fig:NumFailuresVsMemorySize}, for small values of \(t_e\), \alg is unable to find solutions, indicating that the number of samples required exceeds \(t_e\). As \(t_e\) increases beyond $2.8 \times 10^5$, \alg achieves a 100\% success rate, demonstrating that \(t_e\) is not a sensitive tuning parameter with respect to finding solutions.

% As shown in Figure \ref{fig:NumFailuresVsMemorySize}, for small values of \(t_e\), \alg is unable to find a solution, indicating that the number of samples required exceeds \(t_e\). As \(t_e\) increases, \alg achieves a 100\% success rate with values of $2.8 \times 10^5$ or larger, indicating probabilistic completeness.\qh{I don't know if we want to stress the probabilistic completeness aspect here.. maybe we want to say that the performance/efficiency of the algorithm is dependent on $t_e$, but also that it is not a very sensitive hyperparameter because you get high success rate after this value} 

% As depicted in Fig. \ref{fig:RuntimeVsMemorySize}, as \(t_e\) increases, so does the runtime. This increase in runtime is due to several factors. First, when using a larger \(t_e\), \alg on average uses a larger branching factor \(\lambda\), which results in the tree containing more nodes. This trend is shown in Table \ref{table:t_e}, where the number of nodes in the tree increases by approximately \(1 \times 10^5\) for every \(2 \times 10^5\) increase in \(t_e\). Second, since a larger \(t_e\) requires more memory, the implementation of \alg experiences more frequent cache misses, leading to data being fetched from slower global memory more often.

Fig. \ref{fig:RuntimeVsMemorySize} shows an increase in the hyperparameter \(t_e\) also increases the runtime. This can be attributed to two main factors. First, as \(t_e\) increases, \alg typically uses a larger branching factor \(\lambda\), resulting in a tree with more nodes. In this experiment, the number of nodes in the tree increases by approximately \(1 \times 10^5\) for every \(2 \times 10^5\) increase in \(t_e\). Specifically, we observed $3.1 \times 10^5$ nodes when \(t_e = 4 \times 10^5\) and $6.1 \times 10^5$ nodes when \(t_e = 10 \times 10^5\). Second, the larger \(t_e\) demands more memory, leading to more frequent cache misses in the implementation of \alg, which causes data to be fetched from slower global memory more often.

% \np{should I just remove the following paragraph?}
% Finally, Figure~\ref{fig:RuntimeVsMemorySize} demonstrates that, to achieve optimal performance, \alg needs to be tuned for a specific dynamic system. This tuning process is relatively straightforward, as it primarily depends on the dimensionality of the problem. To illustrate that \alg is robust to changes in $t_e$ with a sufficiently large expected tree size, a $t_e$ value of $2 \times 10^5$ is used for all 6-dimensional dynamical systems, and for the 12-dimensional system, $t_e$ is set to $4 \times 10^5$.\qh{this paragraph feels out of place... it is at the end but the results from table 1 is at the start of the experiment section}

% In summary, this numerical experiment shows that \alg is probabilistically complete. 

\begin{figure}[t!]
    \centering
    \begin{subfigure}[b]{0.51\columnwidth}
        \centering
        \includegraphics[width=\textwidth]{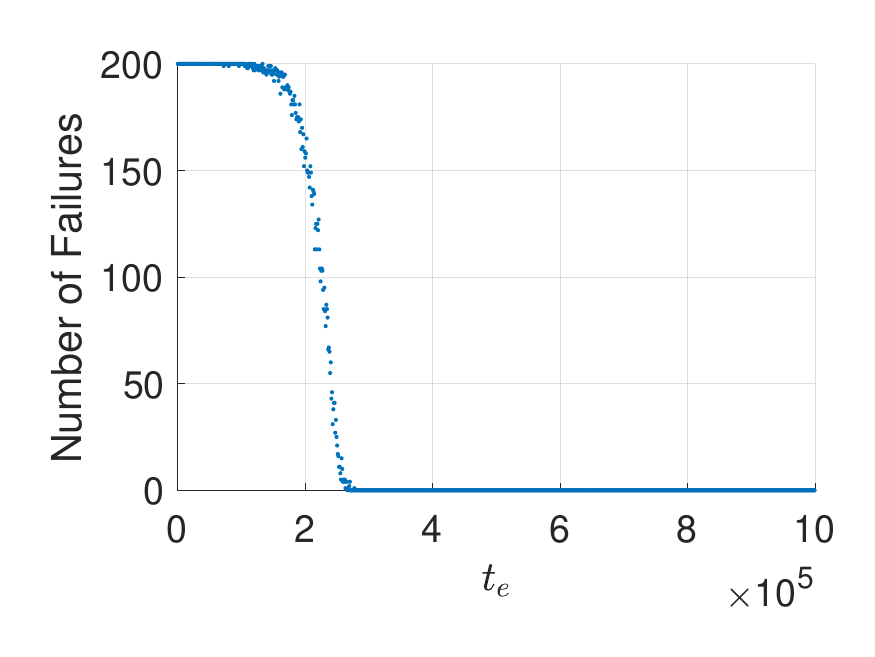}
        \caption{}
        \label{fig:NumFailuresVsMemorySize}
    \end{subfigure}
    % \hfill
    \begin{subfigure}[b]{0.47\columnwidth}
        \centering
        \includegraphics[width=\textwidth]{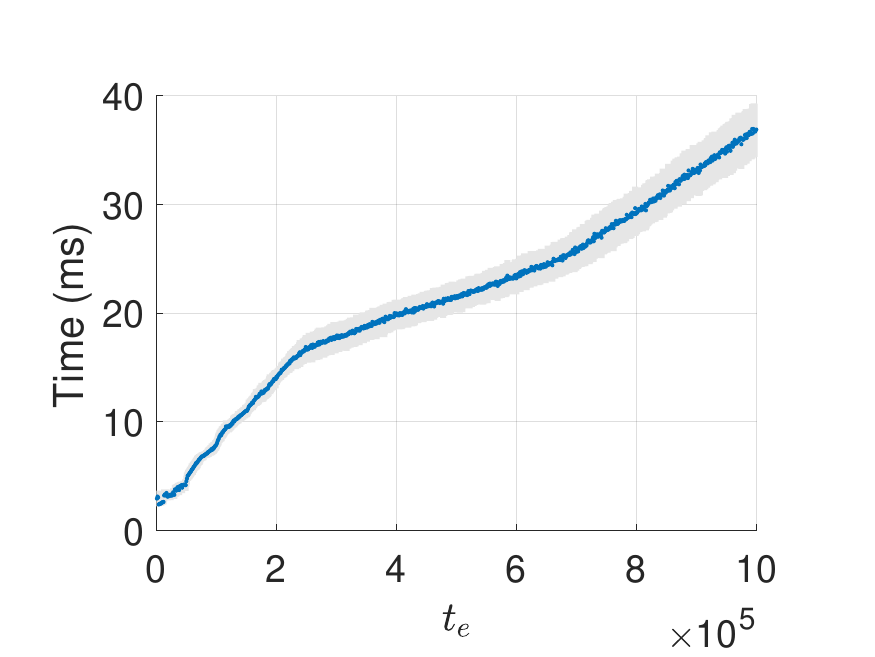}
        \caption{}
        \label{fig:RuntimeVsMemorySize}
    \end{subfigure}
    % \hfill
   \caption{Results of varying the expected tree size $t_e$ for the 12D nonlinear quadcopter system in environment \ref{fig:trees}. (a) Number of Failures vs. $t_e$. (b) Mean runtime and variance of \alg vs. $t_e$.}
    \label{fig:iteration_grid}
\end{figure}

\section{Conclusion}
    \label{sec:conclusion}
    % \np{Introduced and analyzed novel approach to problem \ref{problem}, parallelizing large portions of previously an entirely sequential process, better suiting modern high throughput devices.}

% \np{Allows for potential of planning in real-world changing environments while considering dynamical feasibility for general dynamic systems. }

% \np{Simulated results show that planning times are less than 10 milliseconds for 6-dimensional systems (6DDI and Dubins airplane models). For more complex nonlinear 12-dimensional systems, the planning times are less than 25 milliseconds. These results demonstrate the robustness of \alg in efficiently finding solutions across problems of varying difficulty.}

% \qh{Something about an algorithm that changes $t_e$}

We have introduced a novel motion planning algorithm for kinodynamic systems that enables a significant parallelization of a process previously considered inherently sequential. Our algorithm is well suited to exploit the recent advancements of modern computing devices and is equipped to scale as hardware continues to improve. Benchmark results show planning times of less than $8$ ms for 6-dimensional systems and less than $25$ ms for a 12-dimensional nonlinear system, representing an improvement of up to three orders of magnitude compared to traditional motion planning algorithms.
% For future work, we plan to extend \alg in two directions. First, we aim to make the hyperparameter $t_e$ dynamic, allowing the number of nodes in $\mathcal{T}$ to be resized while maintaining runtime performance. Second, we plan to extend \alg to near-optimal planning, exploring potential speedups for even more challenging problem settings.
% 
For future work, we plan to make the hyperparameter $t_e$ dynamic (adaptive) and extend \alg to a near-optimal planner.

\bibliographystyle{IEEEtran}
\bibliography{refs,references}

\end{document}